\documentclass[11pt]{article}
\usepackage{amsfonts}
\usepackage{mathrsfs}
\usepackage{amsmath,amssymb}
\usepackage{mathtools}

\DeclareMathOperator{\EX}{\mathbb{E}}

\usepackage[latin1]{inputenc}
\usepackage{amsmath,amssymb}

\usepackage{amsthm}
\usepackage{latexsym}
\usepackage{epstopdf}
\usepackage{geometry}  
\usepackage{pict2e,picture}
\usepackage{bigints}  
\usepackage[thinlines]{easytable}
\usepackage{enumerate}
\usepackage{array}
\usepackage{fix-cm}
\usepackage{tikz}
\usetikzlibrary{matrix}

\newcolumntype{C}[1]{>{\centering\arraybackslash}m{#1}}
\newcolumntype{R}[1]{>{\raggedleft\arraybackslash}m{#1}}
 \usepackage{lipsum}
 
 \usepackage[symbol]{footmisc}

\usepackage[active]{srcltx}
\usepackage{graphicx}
\usepackage{epstopdf}
\usepackage{enumitem}   
\usepackage[T1]{fontenc}
\usepackage{amsmath}

\usepackage{footnote}
\usepackage{footmisc}
\makesavenoteenv{tabular}
\makesavenoteenv{table}
\newcommand\numberthis{\addtocounter{equation}{1}\tag{\theequation}}

\usepackage{bbm}

\DeclareMathOperator*{\argmin}{arg\,min}

\usepackage[
bookmarks=true,         
bookmarksnumbered=true, 
colorlinks=true, pdfstartview=FitV, linkcolor=blue, citecolor=blue,
urlcolor=blue]{hyperref}

\newtheorem {theorem}{Theorem}

\newtheorem {corollary}{Corollary}

\newtheorem{definition}{Definition}

\renewcommand\footnotemark{}

\begin{document}
\date{\vspace{-9ex}}

\title{Neural networks with superexpressive activations and integer weights}

\maketitle

\begin{center}
	 Aleksandr Beknazaryan \footnote{abeknazaryan@yahoo.com}

	\bigskip
\end{center}

\begin{abstract}An example of an activation function $\sigma$ is given such that networks with activations $\{\sigma, \lfloor\cdot\rfloor\}$, integer weights and a fixed architecture depending on $d$ approximate continuous functions on $[0,1]^d$. The range of integer weights required for $\varepsilon$-approximation of H\"older continuous functions is derived, which leads to a convergence rate of order $n^{\frac{-2\beta}{2\beta+d}}\log_2n$ for neural network regression estimation of unknown $\beta$-H\"older continuous function with given $n$ samples. 
\end{abstract}
\textbf{Introduction.}
 A family of activation functions $\mathcal{A}$ is called superexpressive if for every input dimension $d$ there are networks of fixed architecture and with activations all belonging to $\mathcal{A}$ that arbitrarily well approximate  functions from $C([0,1]^d)$ in the uniform norm. Several examples of simple superexpressive families of activation functions are given in \cite{Y}. In particular, in \cite{Y}, Theorem 3, it is shown that the families $\{\sin,\arcsin\}$ and $\{\sigma_1, \lfloor\cdot\rfloor\}$ are superexpressive, where $\sigma_1$ is a real analytic function which is non-polynomial on some interval. Although fixing the architecture, the superexpressiveness of activations alone does not tell much about the complexity of approximant networks, which in this case is associated with choices of network weights. Even if we bound the weights required for attaining a given approximation error, the estimation of entropy of networks may still need the activations to be Lipschitz continuous (see, e.g., \cite{SH}, \cite{TXL}). Thus, to bound the entropy of approximant networks we may not only bound the weights but also discretize them. Once this is done, we can directly count the number of networks needed to attain the given approximation error. We will therefore consider networks with weights from $\mathbb{Z}$ and otherwise our construction will be similar to the one presented in the proof of Theorem 3 of \cite{Y}. That proof is based on the density of irrational windings on the torus, which in our case is replaced by an effective Kronecker's theorem. The latter not only assures that integer multipliers are enough to densely cover the torus but also bounds the integers needed to attain a given covering radius. We will consider the family of activation functions $\mathcal{A}=\{\sigma, \lfloor\cdot\rfloor\},$ where the role of the activation $\sigma$ is to guarantee that the conditions of the Kronecker's theorem are satisfied and that it gives a small range for the integer multipliers. Having this range we then bound the entropy of approximant networks and use this bound to get for $\beta$-H\"older continuous regression functions a convergence rate of order $n^{\frac{-2\beta}{2\beta+d}}\log_2n$.
 
  Note that our approach is in some sense opposite to the one given in \cite{B}, where approximations by deep networks with weights $\{0,\pm\frac{1}{2}, \pm 1, 2\}$ are considered: in one case we fix a finite set of weights and adjust the network architecture and in the other case we fix the network architecture and adjust the integer weights to attain a certain approximation rate.\newline\newline
\textbf{An effective Kronecker's Theorem.}
In this part we present an effective version of Kronecker's Theorem given in \cite{V}. To state the theorem we will need the following definitions of absolute values, places and heights on number fields.
\begin{definition}
An absolute value on a number field $K$ is a function $|\cdot|_\nu: K\to\mathbb{R}_+$ satisfying

1. $|x|_\nu=0$ if and only if $x=0$;

2. $|xy|_\nu=|x|_\nu|y|_\nu$ for all $x,y\in K$;

3. $|x+y|_\nu\leq |x|_\nu+|y|_\nu$ for all $x,y\in K$.
\end{definition}
If the third condition above is replaced by a stronger condition $|x+y|_\nu\leq\max\{|x|_\nu, |y|_\nu\},$ then the absolute value $|\cdot|_\nu$ is called \textit{non-archimedean} and otherwise it is called \textit{archimedean}.
\begin{definition}
Two absolute values $|\cdot|_1$ and  $|\cdot|_2$ on $K$ are equivalent if there exists some $\lambda>0$ such that $|\cdot|_1=|\cdot|^\lambda_2$. An equivalence class of absolute values on $K$ is called a place of $K$. The collection of all places of $K$ is denoted by $M_K$.
\end{definition}
Let $\overline{\mathbb{Q}}$ denote the field of algebraic numbers.
\begin{definition}
For $\boldsymbol{\alpha}=(\alpha_1,...,\alpha_N)\in \overline{\mathbb{Q}}^N\setminus\{\normalfont{\textbf{0}}\}$ let $K$ be an extension of the field of rational numbers $\mathbb{Q}$ of degree $[K:\mathbb{Q}]$ such that $\boldsymbol{\alpha}\in K^N$. The number 
$$H(\boldsymbol{\alpha})=\prod\limits_{\nu\in M_K}\bigg(\max(|\alpha_1|_\nu,...,|\alpha_N|_\nu)\bigg)^{1/[K:\mathbb{Q}]}.$$
is called an absolute height of $\boldsymbol{\alpha}$.
\end{definition}
It can be shown that the absolute height is independent of the choice of $K$ (see \cite{V} for proof and more details regarding the above definitions). For $\boldsymbol{\alpha}=(\alpha_1,...,\alpha_N)\in \overline{\mathbb{Q}}^N$ let $r:=[\mathbb{Q}(\alpha_1,...,\alpha_N):\mathbb{Q}]$ be the degree of extension field over rationals generated by $\boldsymbol{\alpha}$. For $\varepsilon>0$ denote
$$Q(\boldsymbol{\alpha}, \varepsilon):=r(N+1)^{2r}(H(1,\alpha_1,...,\alpha_N))^r\bigg(\frac{1}{\varepsilon}\bigg)^{r-1}.$$
 The following is a simplified version of Theorem 3.11 from \cite{V}:
\begin{theorem}
Let $\boldsymbol{\alpha}=(\alpha_1,...,\alpha_N)$ be a vector with algebraic and rationally independent coordinates, that is, 
$$\{\normalfont\textbf{z}\in\mathbb{Q}^N: \normalfont\textbf{z}^\intercal\cdot\boldsymbol{\alpha}\in\mathbb{Q}\}=\{\textbf{0}\}.$$ 
Then for every $\varepsilon>0$ and every $(b_1,...,b_N)\in[0,1)^N$ there is $q\in \mathbb{Z}$ with $|q|\leq Q(\boldsymbol{\alpha}, \varepsilon)$ such that
$$|\phi(q\alpha_i)-b_i|\leq\varepsilon,\quad i=1,...,N,$$
where $\phi(x)=x-\lfloor x\rfloor$.
\end{theorem}

As the choice of the activation function $\sigma$ in the next part suggests, we will be interested in application of the above theorem to the case $\boldsymbol{\alpha}=(2^{1/(N+1)}, 2^{2/(N+1)},...,2^{N/(N+1)})$. In this case we have that $r=N+1$ and, therefore, there are at most $N+1$ archimedean places on $\mathbb{Q}(\boldsymbol{\alpha})$ (see \cite{BG}, Subsection 1.3.8). Also, as the non-archimedean absolute values of integers are in $[0,1]$ (\cite{S}, Lemma 6A), then  
\begin{align*}
\begin{split}
&\bigg(H(1,(2^{1/(N+1)}, 2^{2/(N+1)},...,2^{N/(N+1)})\bigg)^{N+1}=\\
&\prod\limits_{\nu\in M_{\mathbb{Q}(\boldsymbol{\alpha})}}\max(1, |2^{1/(N+1)}|_\nu,...,|2^{N/(N+1)}|_\nu)=\prod\limits_{\nu\in M_{\mathbb{Q}(\boldsymbol{\alpha})}}\max(1, |2|^{1/(N+1)}_\nu,...,|2|^{N/(N+1)}_\nu)=\\
&\prod\limits_{\nu\in M_{\mathbb{Q}(\boldsymbol{\alpha})}}\max(1, |2|^{N/(N+1)}_\nu)= \prod\limits_{\substack{\nu\in M_{\mathbb{Q}(\boldsymbol{\alpha})} \\ \nu\; \textrm{archimedean}}}\max(1, |2|^{N/(N+1)}_\nu)\leq2^{N}.
\end{split}
\end{align*}
\newpage\noindent We thus get the following 
\begin{corollary}\label{cor}
For every $\varepsilon>0$ and every $(b_1,...,b_N)\in[0,1)^N$ there is $q\in \mathbb{Z}$ with 
$$|q|\leq (N+1)^{2N+3}\bigg(\frac{2}{\varepsilon}\bigg)^{N}$$ such that
$$|\phi(q2^{i/(N+1)})-b_i|\leq\varepsilon,\quad i=1,...,N.$$
\end{corollary}
\noindent\textbf{Network selection and approximation.}
For a set of $p$ functions $g^{1},...,g^{p}:\mathbb{R}\to\mathbb{R}$ and two sets of $p$ numbers $\{v_1,...,v_p\}, \{y_1,...,y_p\}\in\mathbb{R}$ define 

\[\begin{pmatrix}
g_{{v_1}}^{1}
\\
\vdots
\\
g_{v_p}^{p}
\end{pmatrix}\begin{pmatrix}
y_1
\\
\vdots
\\
y_p
\end{pmatrix}= \begin{pmatrix}
g^{1}(y_1+v_1)
\\
\vdots
\\
g^{p}(y_p+v_p)
\end{pmatrix}.\]
For $K, M\in\mathbb{N}$ and $q\in\mathbb{Z}$ denote $k_M:=((M+1)^d-1)(M+1)^d/2$ and consider a feedforward neural network $Z^d_{K, M, q}$ on $[0,1]^d$ of the form 
\begin{align*}
&Z^d_{K, M, q}(\textbf{x})=Z^d_{K, M, q}(x_1, x_2,...,x_d)=\\
&(2Kq, -2K)
\begin{pmatrix}
\sigma_{k_M}
\\
\lfloor \cdot \rfloor_0
\end{pmatrix}
\begin{pmatrix}
1 & 0
\\
0 & q
\end{pmatrix}
\begin{pmatrix}
\lfloor \cdot \rfloor_0
\\
\sigma_{k_M}
\end{pmatrix}
\begin{pmatrix}
1
\\
1
\end{pmatrix}
\lfloor \cdot \rfloor_1
(1, M+1,..., (M+1)^{d-1})
\begin{pmatrix}
\lfloor \cdot \rfloor_0
\\
\lfloor \cdot \rfloor_0
\\
\vdots
\\
\lfloor \cdot \rfloor_0
\end{pmatrix}
(M\cdot I_d)
\begin{pmatrix}
x_1
\\
x_2
\\
\vdots
\\
x_d
\end{pmatrix}-K,\end{align*}
where $I_d\in\mathbb{R}^{d\times d}$ is an identity matrix, $\lfloor \cdot \rfloor$ is the floor function and $\sigma:\mathbb{R}\to\mathbb{R}$ is defined as 
\[
\sigma(x)= \begin{cases} 2^{\frac{x-(m-1)m/2}{m+1}}, & x, m\in \mathbb{N},\; \frac{(m-1)m}{2}<x\leq\frac{m(m+1)}{2}, \\
0, & x\in \mathbb{R}\setminus\mathbb{N}.
\end{cases}
\]
In fact, the values of $\sigma$ on $\mathbb{R}\setminus\mathbb{N}$ will not play a role and can thus be defined arbitrarily. Here are the first few nonzero values of $\sigma$:\newline
$\sigma(1)=2^{1/2};$\newline
$\sigma(2)=2^{1/3}, \; \sigma(3)=2^{2/3};$\newline
$\sigma(4)=2^{1/4}, \; \sigma(5)=2^{2/4}, \; \sigma(6)=2^{3/4}.$\newline

Note that analytically we can write the network $Z^d_{K, M, q}$ as 
$$Z^d_{K, M, q}(\textbf{x})=2K\phi(q\sigma(k_M+g_M(\textbf{x})))-K,$$
where $\phi(x)=x-\lfloor x\rfloor$ and $g_M(\textbf{x})=1+\sum\limits_{k=1}^d(M+1)^{k-1}\lfloor Mx_k\rfloor$.
For $Q\in\mathbb{N}$ define a set of networks

$$\mathcal{Z}_{K,M}^d(Q):=\{Z^d_{K, M, q}, |q|\leq Q\}.$$
For $\beta,F\in\mathbb{R}_+$ and $K\in\mathbb{N}$ define
\begin{align*}
\mathcal{H}^\beta_d(F, K)=\bigg\{f:[0,1]^d\to\mathbb{R}: \|f\|_\infty< K \textrm{ and } |f(\textbf{x})-f(\textbf{y})|\leq F|\textbf{x}-\textbf{y}|_\infty^{\beta} \textrm{ for all } \textbf{x},\textbf{y}\in[0,1]^d \bigg\}.
\end{align*}
\newpage\noindent We have the following 
\begin{theorem}\label{appr}
For any $\varepsilon>0$ and any $f\in\mathcal{H}^\beta_d(F, K)$ there is a network $\normalfont Z^d_{K, M, q}(\textbf{x})\in\mathcal{Z}_{K,M}^d(Q)$ with $M=\lceil (2F/\varepsilon)^{1/\beta}\rceil$ and $$Q=\bigg\lceil(N+1)^{2N+3}\bigg(\frac{8K}{\varepsilon}\bigg)^{N}\bigg\rceil,$$
where $N=(M+1)^d$,
such that $\normalfont\|Z^d_{K, M, q}(\textbf{x})-f(\textbf{x})\|_\infty\leq\varepsilon$.
\end{theorem}
\begin{proof}
Take any $\varepsilon>0$ and let $M=\lceil (2F/\varepsilon)^{1/\beta}\rceil$. Then the function 
$$g_M(\textbf{x})=g_M(x_1,...,x_d)=1+\sum\limits_{k=1}^d(M+1)^{k-1}\lfloor Mx_k\rfloor$$ 
from $[0,1]^d$ to $[1, (M+1)^d]\cap\mathbb{Z}$ maps each of $(M+1)^d$ sets $$I_{M,\textbf{m}}:=\bigg([\frac{m_1}{M}, \frac{m_1+1}{M})\times...\times[\frac{m_d}{M}, \frac{m_d+1}{M})\bigg)\cap[0,1]^d,$$ $\textbf{m}=(m_1,...,m_d)\in[0,M]^d\cap\mathbb{Z}^d,$ to a unique integer from $[1, (M+1)^d]$. Denote $N=(M+1)^d$ and let $J_1,...,J_N$ be the enumeration of the sets $I_{M,\textbf{m}}, \textbf{m}\in[0,M]^d\cap\mathbb{Z}^d,$ such that $g_M(\textbf{x})=i$ for $\textbf{x}\in J_i,$ $i=1,...,N$. Take any set of $N$ points $\textbf{y}_i\in J_i$ and denote $$b_i=\frac{f(\textbf{y}_i)+K}{2K}\in[0,1)^N,\quad i=1,...,N.$$ 
Let $k_M:=((M+1)^d-1)(M+1)^d/2=\frac{(N-1)N}{2}$. As $\sigma(k_M+i)=2^{i/(N+1)},$ then, by Corollary \ref{cor}, there exists $q\in\mathbb{Z}$ with $$|q|\leq (N+1)^{2N+3}\bigg(\frac{8K}{\varepsilon}\bigg)^{N}$$ such that
$$|\phi(q\sigma(k_M+i))-b_i|\leq\frac{\varepsilon}{4K},\quad i=1,...,N.$$
Thus,
$$|2K\phi(q\sigma(k_M+i))-K-f(\textbf{y}_i)|\leq\frac{\varepsilon}{2},\quad i=1,...,N,$$
and, therefore, for the network 
$$Z^d_{K, M, q}(\textbf{x})=2K\phi(q\sigma(k_M+g_M(\textbf{x})))-K$$
we have that for $\textbf{x}\in J_i$
\begin{align*}
\begin{split}
&|Z^d_{K, M, q}(\textbf{x})-f(\textbf{x})|=|2K\phi(q\sigma(k_M+g_M(\textbf{x})))-K-f(\textbf{x})|\\
&\leq |2K\phi(q\sigma(k_M+i))-K-f(\textbf{y}_i)|+|f(\textbf{y}_i)-f(\textbf{x})|\\
&\leq\frac{\varepsilon}{2}+F|\textbf{y}_i-\textbf{x}|^\beta_\infty\leq\varepsilon.
\end{split}
\end{align*}
As $[0,1]^d=\cup_{i=1}^NJ_i,$ then $\|Z^d_{K, M, q}(\textbf{x})-f(\textbf{x})\|_\infty\leq\varepsilon$.
\end{proof}\noindent\textbf{Application to nonparametric regression}. Let $f_0\in\mathcal{H}^\beta_d(F, K)$ be an unknown regression function and let $(\textbf{X}_i, Y_i)$, $i=1,...,n,$ be $n$ observed iid pairs following a regression model 
$$Y_i=f_0(\textbf{X}_i)+\epsilon_i,$$
where the standard normal noise variables $\epsilon_i$ are assumed to be independent of $\textbf{X}_i$. Our goal is to choose appropriate $M_n, Q_n\in\mathbb{N}$ so that the empirical risk minimizer $$\hat{Z}_n\in\argmin_{Z\in\mathcal{Z}_{K,M_n}^d(Q_n)}\sum_{i=1}^{n}(Y_i-Z(\textbf{X}_i))^2$$ can well approximate $f_0$. The accuracy of approximation of $f_0$ by the estimator $\hat{Z}_n$ is measured by the prediction error
$$R(\hat{Z}_n,f_0)=\EX_{f_0}[(\hat{Z}_n(\textbf{X})-f_0(\textbf{X}))^2],$$
where $\textbf{X}\stackrel{\mathcal{D}}{=}\textbf{X}_1$ is independent of the sample $(\textbf{X}_i, Y_i)$ and the subscript $f_0$ indicates that the expectation is taken over the training data generated by the regression model. 

Choose  $M_n=\lceil (2F)^{\frac{1}{\beta}}n^{\frac{1}{2\beta+d}}\rceil$ and $$Q_n=\bigg\lceil(N_n+1)^{2N_n+3}\bigg(8Kn^{\frac{\beta}{2\beta+d}}\bigg)^{N_n}\bigg\rceil,$$ where $N_n:=(M_n+1)^d.$
From \cite{SH}, Lemma 4, it follows that for any $\delta\in(0,1]$ 
\begin{align*}
&R(\hat{Z}_n,f_0)\leq\\
&4\bigg[\inf\limits_{Z\in{\mathcal{Z}_{K,M_n}^d(Q_n)}}\EX[\normalfont(Z(\textbf{X})-f_0(\textbf{X}))^2]+K^2\frac{18\log_2\mathcal{N}(\delta,\mathcal{Z}_{K,M_n}^d(Q_n),\|\cdot\|_\infty)+72}{n}+32\delta K\bigg], \numberthis \label{R}
\end{align*}
where $\mathcal{N}(\delta,\mathcal{Z}_{K,M_n}^d(Q_n),\|\cdot\|_\infty)$ is the covering number of $\mathcal{Z}_{K,M_n}^d(Q_n)$ of radius $\delta$ taken with respect to the $\|\cdot\|_\infty$ distance of functions on $[0,1]^d$. As there are only $2Q_n+1$ networks in $\mathcal{Z}_{K,M_n}^d(Q_n)$, then for any $\delta>0$
$$\log_2\mathcal{N}(\delta,\mathcal{Z}_{K,M_n}^d(Q_n),\|\cdot\|_\infty)\leq\log_2\mathcal{N}(0,\mathcal{Z}_{K,M_n}^d(Q_n),\|\cdot\|_\infty)\leq\log_2(2Q_n+1)\leq C'n^{\frac{d}{2\beta+d}}\log_2n,$$
for some constant $C'=C'(\beta, d, F, K)$. As $f_0\in\mathcal{H}^\beta_d(F, K)$, then, applying Theorem \ref{appr} with $\varepsilon=n^{-\frac{\beta}{2\beta+d}}$, we get that 
\begin{align*}
\inf\limits_{Z\in{\mathcal{Z}_{K,M_n}^d(Q_n)}}\EX[\normalfont(Z(\textbf{X})-f_0(\textbf{X}))^2]\leq n^{\frac{-2\beta}{2\beta+d}}.
\end{align*}
 Thus, from \eqref{R} we get an existence of a constant $C=C(\beta, d, F, K)$ such that 
\begin{align*}
R(\hat{Z}_n,f_0)\leq Cn^{\frac{-2\beta}{2\beta+d}}\log_2n,
\end{align*}
which coincides, up to a logarithmic factor, with the minimax estimation rate $n^{\frac{-2\beta}{2\beta+d}}$ of the prediction error for $\beta$-smooth functions.

\end{document}